\documentclass{svproc}

\usepackage{courier}
\usepackage[english]{babel}
\usepackage[utf8x]{inputenc}
\usepackage[mathscr]{euscript}

\usepackage{amsmath}
\usepackage{amssymb}
\usepackage{mathtools}

\usepackage{twoopt}
\usepackage{verbatim}
\usepackage{xparse}
\usepackage[usenames, dvipsnames]{color}
\usepackage{cite}
\usepackage[caption=false]{subfig}
\definecolor{mhcolor}{rgb}{0.0,0.5,0.0}

\newcommand{\ve}[1]{\vec{#1}}

\newcommand{\Manip}[0]{m}
\newcommand{\Obj}[0]{o}

\newcommand{\Jac}[1][]{\ve{J}_{#1}}
\newcommand{\NJac}[1][]{\ve{J}_{N,#1}}
\newcommand{\TJac}[1][]{\ve{J}_{T,#1}}
\newcommand{\Fvar}[1][]{\ve{\lambda}_{#1}}
\newcommand{\FNet}[1][]{\ve{F}_{#1}}
\newcommand{\FB}[1][]{\ve{B}_{#1}}
\newcommand{\FBScale}[1][]{c_{#1}}
\newcommand{\FMMat}[1][]{\ve{A}_{#1}}

\newcommand{\FMMap}[1][]{H_{#1}}
\newcommand{\Fn}[1][]{\ve \lambda_{N_{#1}}}
\newcommand{\Ft}[1][]{\ve \lambda_{T_{#1}}}

\newcommand{\Di}[1][]{\ve \delta_{#1}}
\newcommand{\SVar}[1][]{\ve \gamma_{#1}}
\newcommand{\MuMat}[1][]{\ve \mu_{#1}}

\newcommand{\GVel}[1][]{\ve{V}_{#1}}
\newcommand{\GVelJac}[1][]{\tilde{\ve{R}}_{#1}}
\newcommand{\EVec}[0]{\ve e}
\newcommand{\EMat}[0]{\ve E}
\newcommand{\FSet}{\mathcal{F}}

\newcommand{\NumC}[0]{k}
\newcommand{\DimManip}[0]{n}
\newcommand{\Centroid}[1][]{\ve{p}_{#1}}

\newcommand{\Dist}[1][]{\ve{\phi}_{#1}}

\newcommand{\TStart}{-}
\newcommand{\TEnd}{+}
\newcommand{\TSChange}{\ve{\Delta}}
\newcommand{\DFn}[1][]{\ve \Lambda_{N_{#1}}}
\newcommand{\DFt}[1][]{\ve \Lambda_{T_{#1}}}
\newcommand{\DDist}[1][]{\TSChange \Dist[#1]}
\newcommand{\DRBPos}[1][]{\TSChange \RBPos[#1]}

\newcommand{\Seq}[3][\Natural]{\left( {#2}_{#3} \right)_{#3\in#1}}

\newcommand{\parenEnc}[1]{\left( #1 \right)}
\newcommand{\Enorm}[1]{\left\lVert#1\right\rVert_2}

\newcommand{\Real}[0]{\mathbb{R}}

\newcommand{\ZVec}[0]{\ve{0}}
\newcommand{\ZMat}[0]{\ZVec}
\newcommand{\Eye}[0]{\ve I}

\newcommand{\PVDiff}[2]{\frac{\ve{\partial}#1 }{\ve{\partial}#2}}

\newcommand{\ArbVec}[0]{\ve{x}}

\newcommand{\ArbRCon}[0]{c}
\newcommand{\ArbRConB}[0]{r}
\newcommand{\ArbEVal}[0]{\lambda}
\newcommand{\ArbIter}[0]{i}
\newcommand{\ArbIterB}[0]{j}

\newcommand{\ArbSize}[0]{n}

\newcommand{\LCPMat}[0]{\ve M}
\newcommand{\LCPVec}[0]{\ve w}
\newcommand{\LCPVar}[0]{\ve z}
\newcommand{\LCP}[2]{\mathrm{LCP}(#1, #2)}
\newcommand{\LCPSol}[2]{\mathrm{SOL}(#1, #2)}
\newcommand{\LCPSet}[2]{\LCPSol{#1}{#2}}
\newcommand{\LCPWP}[1][]{\LCPMat \LCPVar + \LCPVec \geq \ZVec}
\newcommand{\LCPVarP}[1][]{\LCPVar \geq \ZVec}
\newcommand{\LCPComp}[1][]{\LCPVar^T (\LCPMat \LCPVar + \LCPVec) = \ZVec}

\newcommand{\VCompl}[2]{\ZVec \leq #1 \perp  #2 \geq \ZVec}
\newcommand{\LCPCons}[1][]{\VCompl{\LCPVar}{\LCPMat \LCPVar + \LCPVec}}

\newcommand{\RBVar}[0]{q}

\newcommand{\RBVel}[1][]{\ve{\dot{\RBVar}}_{#1}}
\newcommand{\RBPos}[1][]{\ve{\RBVar}_{#1}}

\newcommand{\RBDVel}[1][]{\ve{u}_{#1}}

\newcommand{\VNLC}[1][]{\VCompl{\NJac[\Obj] \RBVel[\Obj] + \NJac[\Manip] \RBVel[\Manip]}{\Fn}}
\newcommand{\VTLC}[1][]{\VCompl{\TJac[\Obj] \RBVel[\Obj] + \TJac[\Manip] \RBVel[\Manip] + \EMat \SVar}{\Ft}}
\newcommand{\VSLC}[1][]{\VCompl{\MuMat[#1]\Fn - \EMat^T\Ft}{\SVar}}

\newcommand{\PNLC}[1][]{\VCompl{\Dist[]^\TEnd}{\DFn}}
\newcommand{\PTLC}[1][]{\VCompl{\TJac[\Obj] \DRBPos[\Obj] + \TJac[\Manip] \DRBPos[\Manip] + \EMat \SVar}{\DFt}}
\newcommand{\PSLC}[1][]{\VCompl{\MuMat[#1]\DFn - \EMat^T\DFt}{\SVar}}

\newcommandtwoopt{\VLCPInfMat}[2][][]{\begin{bmatrix}
\NJac[#1] \FMMat \NJac[#1]^T & \NJac[#1] \FMMat \TJac[#1]^T & \ZVec \\
\TJac[#1] \FMMat \NJac[#1]^T & \TJac[#1] \FMMat \TJac[#1]^T & \EMat \\
\MuMat & - \EMat & \ZVec
\end{bmatrix}}
\DeclareDocumentCommand\VLCPFinMat { O{\Manip} O{\Obj} O{} } {\VLCPInfMat[#2][#3] + \begin{bmatrix}
\Jac[#1] \FBScale \FB[#3] \Jac[#1]^T & \ZVec \\
\ZVec & \ZVec
\end{bmatrix}}

\newcommandtwoopt{\VLCPVec}[2][\Manip][]{\begin{bmatrix}
\NJac[#1]\RBDVel \\
\TJac[#1]\RBDVel \\
\ZVec
\end{bmatrix}}

\newcommandtwoopt{\PLCPInfMat}[2][][]{\begin{bmatrix}
\NJac[#1]^\TStart \FMMat {\NJac[#1]^{\TStart T}} & \NJac[#1]^\TStart \FMMat {\TJac[#1]^{\TStart T}} & \ZVec \\
\TJac[#1]^\TStart \FMMat {\NJac[#1]^{\TStart T}}  & \TJac[#1]^\TStart \FMMat {\TJac[#1]^{\TStart T}}  & \EMat \\
\MuMat & - \EMat & \ZVec
\end{bmatrix}}

\DeclareDocumentCommand\PLCPFinMat { O{\Manip} O{\Obj} O{} } {\PLCPInfMat[#2][#3] + \begin{bmatrix}
\Jac[#1]^\TStart \FBScale \FB[#3] {\Jac[#1]^{\TStart T}}  & \ZVec \\
\ZVec & \ZVec 
\end{bmatrix}}

\newcommandtwoopt{\PLCPVec}[2][\Manip][]{\begin{bmatrix}
\NJac[#1]^\TStart h\RBDVel^\TStart + \Dist^\TStart\\
\TJac[#1]^\TStart h\RBDVel^\TStart \\
\ZVec
\end{bmatrix}}

\begin{document}
\mainmatter

\title{A Quasi-static Model and Simulation Approach for Pushing, Grasping, and Jamming}
\titlerunning{Quasi-static Pushing, Grasping, and Jamming}
\toctitle{Quasi-static Pushing, Grasping, and Jamming}
\author{Mathew Halm \and Michael Posa}
\institute{GRASP Laboratory \\ University of Pennsylvania \\ Philadelphia, PA, 19104 USA \\ \email{\{mhalm, posa\}@seas.upenn.edu}}

\maketitle
\begin{abstract}

Quasi-static models of robotic motion with frictional contact provide a computationally efficient framework for analysis and have been widely used for planning and control of non-prehensile manipulation.
In this work, we present a novel quasi-static model of planar manipulation that directly maps commanded manipulator velocities to object motion.
While quasi-static models have traditionally been unable to capture grasping and jamming behaviors, our approach solves this issue by explicitly modeling the limiting behavior of a velocity-controlled manipulator.
We retain the precise modeling of surface contact pressure distributions and efficient computation of contact-rich behaviors of previous methods and additionally prove existence of solutions for any desired manipulator motion.
We derive continuous and time-stepping formulations, both posed as tractable Linear Complementarity Problems (LCPs).

\end{abstract}
\section{Introduction}
As frictional contact is the fundamental driving process by which many robots are able to interact with their surroundings, it is unsurprising that its behavior is central to a large body of robotic locomotion and manipulation research (e.g. \cite{ref:Hogan2018,ref:Lynch1992,ref:Posa2014,ref:Posa2016,ref:Song2005,ref:Trinkle1989}).
However, dynamical models of these systems are inherently complex and challenging to simulate and analyze.
Impacts between rigid bodies induce instantaneous jumps in velocity states and a combinatorial explosion of hybrid modes that in conjunction render application of common tools from control theory and trajectory optimization difficult. While there has been notable progress in planning through unknown contact sequences with full dynamics \cite{ref:Posa2014,ref:Mordatch2015,ref:Manchester2017}, model complexity has thus far still inhibited real-time usage.

Many applications in robotics involving frictional contact exhibit structure that permits partial or full circumvention of these difficulties. Particularly, we examine planar tabletop manipulation, where a manipulator effects motion of an object that rests upon a flat, frictional surface. Several results in simulation, control, and planning for such systems have been enabled by quasi-static assumptions---that if manipulator accelerations and velocities are low, a force-balance equation can approximate Netwon's second law. This assumption enables reduced-order modeling of the object's movement in response to contact with a manipulator driven at a particular velocity; such models also often circumvent the complexity associated with state discontinuities in dynamical approaches. Furthermore, quasi-static models often eliminate numerical sensitivity induced by the stiff dynamics of manipulators with high feedback gain controllers.

The tractability of these models has enabled impressive results in formal control analysis, task planning, and learning (e.g. \cite{ref:Dogar2012, ref:Kloss2017, ref:Lynch1996}).
However, the range of motion that these methods are currently able to model is limited. They are often restricted to pushing and non-prehensile motions and are completely unable to usefully express grasping or jamming; in these cases, their associated mathematical programs often yield no solutions or ambiguous behavior. 
Grasping and jamming objects is crucial for a wide range of robot tasks, and much work has been devoted to planning and controlling action before and after grasping events (e.g.\cite{ref:Rodriguez2012,ref:Rimon1996,ref:Hass-Heger2018,ref:Zhou2017b,ref:Paolini2014}); However, much of this work can only describe static grasp configurations, and is unable to depict grasp-like behavior with sliding contacts. The process of acquiring such a grasp itself often involves jamming (e.g. when reorienting an object by pushing it up against a wall), which neither the prehensile nor static grasping models can capture. We therefore find great value in the formulation of a unified quasi-static model that can smoothly capture a complete task involving the acquisition and use of a grasp. The ambiguity in traditional approaches arises from the inconsistent assumptions of rigid bodies and perfect control of manipulator velocity. Our key insight is that by appropriately representing the manipulator's internal controller, physically-grounded motion of the object-manipulator system is guaranteed to exist. We contribute both instantaneous velocity and time-stepping position models that formulate this behavior as linear complementarity problems, and prove existence of solutions for each.

\section{Related Work}

There is a significant body of research that examines manipulation from a quasi-static perspective  \cite{ref:Mason1986,ref:Trinkle1989,ref:Lynch1992,ref:Song2005,ref:Pang2018}. For systems in which the object experiences frictional support, relevant research typically examines the pressure distribution supporting the object. Some earlier works provide guaranteed properties of the object's motion without full knowledge of this distribution.
Mason \cite{ref:Mason1986} derived the voting theorem to construct a mapping from the center of pressure to the direction of the object's angular velocity. Lynch and Mason \cite{ref:Lynch1996} later performed stability and controllability analysis for a manipulator pushing an object with multiple fingers.
Other works alternatively contribute models that directly map manipulator joint velocities to object motion. Trinkle \cite{ref:Trinkle1989} characterized vertical planar manipulation using a nonlinear mathematical program that explicitly solved for the contact forces between the object and the manipulator. A similar, more general model for arbitrary 3D rigid multibody systems was proposed in Trinkle et al. \cite{ref:Trinkle2005}. Neither formulation can model detailed pressure distributions between surfaces, as they model friction as acting at a finite set of points.

Efficient and expressive modeling of the complex behavior of these pressure distributions was enabled by Goyal et al. \cite{ref:Goyal1991}, who define the \emph{limit surface}, the bounded convex set of friction loads that the surfaces in contact might exert on each other. Zhou et al. \cite{ref:Zhou2018} approximate this set in high detail as a semialgebraic set fitted to experimental data, and from it derive a model for the motion of the manipulator-object system.
This method, however, assumes that the manipulator follows a commanded velocity exactly, and therefore does not return a solution for infeasible commands. As such commands often result in grasping in a full dynamical model, valuable behaviors are not captured. Some planning methods (e.g \cite{ref:Erdmann1986}) introduce manipulator compliance through the generalized dampers of Whitney \cite{ref:Whitney1977}, though they have not been incorporated into  models capable of initiating and releasing multiple contacts. 

Pang and Tedrake \cite{ref:Pang2018} also devise a resolution for non-existence in velocity-controlled 3D rigid quasi-static systems. They model deviations from desired velocity as a result of local elastic deformation at point contacts, and preserve realism by minimizing them in a mixed-integer quadratic program (MIQP) formulation. While their work applies to a more general class of systems, our method has three key advantages for planar manipulation: we draw model behavior from a problem class that is far more tractable than MIQPs; our proof of existence makes possible formal guarantees for controller performance as well as simulation reliability; and our inclusion of a limit-surface model allows for more realistic modeling without introducing significant complexity.
\section{Background}

\begin{figure}[h]
\includegraphics[width=7cm]{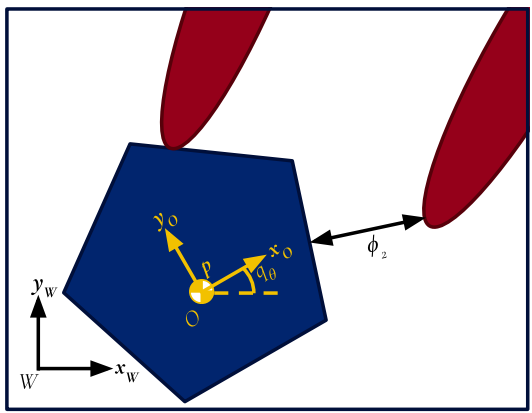}
\centering
\caption{An example of the type of system described in this section. The blue pentagon represents the object, while the red shapes represent the manipulator.}\label{fig:Overhead}
\end{figure}
We now describe the behavior of a frictionally-supported object with center of mass $\Centroid$ in contact with a manipulator, depicted in Figure \ref{fig:Overhead}. For a more detailed treatment of rigid-body contact dynamics, we refer the reader to \cite{ref:Anitescu2002}. Coordinates of the object and manipulator configurations are denoted $\RBPos[\Obj] = \begin{bmatrix}
\Centroid^W; & q_\theta
\end{bmatrix} \in \Real^3$ and $\RBPos[\Manip] \in \Real^\DimManip$, respectively. The manipulator is assumed to be controlled to track commanded velocity $\RBDVel$. Frictional contact with the manipulator causes the object to move with body-axis velocity $\GVel = \begin{bmatrix}
v_{x,b} & v_{y,b} & \dot{q}_\theta
\end{bmatrix}^T$, which can be converted to the world frame velocity $\RBVel[\Obj]$ with the full-rank transformation $\RBVel[\Obj] = \GVelJac(q_\theta)\GVel$. Contact is modeled as normal and frictional forces $\Fn \in \Real^\NumC$ and $\Ft \in \Real^{2\NumC}$ at $\NumC$ pairs of points on the boundaries of the object and manipulator. The coefficients of friction at these points are given by the diagonal matrix $\MuMat \in \Real^{\NumC \times \NumC}$. Each contact's tangential force is split into two non-negative opposing forces $\Ft[\ArbIter]^+$ and $\Ft[\ArbIter]^-$ (as in \cite{ref:Anitescu2002}), such that the net tangential force is equal to $\Ft[\ArbIter]^+ - \Ft[\ArbIter]^-$. The distances between each pair of points are represented by the vector $\Dist[](\RBPos[\Obj],\RBPos[\Manip]) \in \Real^\NumC$. When the points are in contact, their separating velocity can be calculated from the generalized velocities as $\dot{\Dist[]} = \NJac[\Obj] \RBVel[\Obj] + \NJac[\Manip] \RBVel[\Manip]$ using Jacobian matrices $\NJac[\Obj] = \PVDiff{\Dist[]}{\RBPos[\Obj]} $ and $\NJac[\Manip] = \PVDiff{\Dist[]}{\RBPos[\Manip]}$. Similarly, there exist Jacobian matrices $\TJac[\Obj]$ and $\TJac[\Manip]$ such that the velocity at which the bodies slide against each other is given by $\TJac[\Obj] \RBVel[\Obj] + \TJac[\Manip] \RBVel[\Manip]$. For notational convenience, we will often group the normal and tangential terms as $\Jac[\Obj] = [\NJac[\Obj] ; \TJac[\Obj]]$, $\Jac[\Manip] = [\NJac[\Manip] ; \TJac[\Manip]]$, and $\Fvar = \begin{bmatrix}
	\Fn; \Ft
\end{bmatrix}$. Using the relationship between the contact velocities and the generalized velocities, the contact forces can be used to determine the forces acting on the object and manipulator as $\FNet[\Obj] = \Jac[\Obj]^T \Fvar $ and $\FNet[\Manip] = \Jac[\Manip]^T \Fvar$, respectively.

We also make use of the constants $\EVec = \begin{bmatrix} 1& 1\end{bmatrix}^T$ and the block-diagonal matrix
\begin{equation}
	\EMat = \begin{bmatrix}
		\EVec & \ZVec & \cdots \\
		\ZVec & \EVec & \\
		\vdots & & \ddots
	\end{bmatrix} \in \Real^{2\NumC \times \NumC}\;.
\end{equation}

\subsection{Linear Complementarity Problems}
Throughout this work, we will make regular use of linear complementarity problems (LCPs). An LCP is a particular type of mathematical program for which solutions can be efficiently computed. LCPs have been widely used by the dynamics and robotics communities for describing the effects of contact (e.g. \cite{ref:Stewart1996,ref:Anitescu2002,ref:Zhou2018}). Here, we briefly introduce the problem formulation and some useful properties, and we refer the reader to \cite{ref:Cottle2009} for a more complete description.

\begin{definition}\label{def:LCPDef}
The \textbf{linear complementarity problem} $\LCP{\LCPMat}{\LCPVec}$ for the matrix $\LCPMat \in \Real^{\ArbSize \times \ArbSize}$ and vector $\LCPVec \in \Real^\ArbSize$ describes the mathematical program
\begin{align}
& {\text{find}}
& & \LCPVar \in \Real^\ArbSize\;,\\
& \text{subject to}
& & \LCPComp\;, \label{eq:lcpcompbeg}\\
& & & \LCPVarP\;, \\
& & & \LCPWP\;, \label{eq:lcpcompend}
\end{align}
for which the set of solutions is denoted $\LCPSet{\LCPMat}{\LCPVec}$. Constraints (\ref{eq:lcpcompbeg})--(\ref{eq:lcpcompend}), called complementarity constraints, are often abbreviated as $\LCPCons$. 
\end{definition}

Note that vector inequalities in the above definition, as well as elsewhere in this work, are taken element-wise. We will find that for LCPs related to frictional behavior, the matrix parameter $\LCPMat$ is often copositive (i.e. $\ArbVec^T \LCPMat \ArbVec \geq 0$ for all $\ArbVec \geq \ZVec$). This property is often theoretically useful, as Corollary 4.4.12 of \cite{ref:Cottle2009}, reproduced below, gives a sufficient condition for copositive LCP feasibility.
\begin{proposition}\label{prop:LCPCopExist}
Let $\LCPVec \in \Real^\ArbSize$, and let $\LCPMat \in \Real^{\ArbSize \times \ArbSize}$ be copositive. Suppose that for every $\LCPVar \in \LCPSol{\LCPMat}{\ZVec}$, we have $\LCPVar^T\LCPVec \geq 0$. It follows that $\LCPSet{\LCPMat}{\LCPVec} \neq \emptyset$, and an element of $\LCPSol{\LCPMat}{\LCPVec}$ can be discovered by Lemke's Algorithm in finite time.
\end{proposition}

\subsection{Friction at Point Contacts Between Object and Manipulator}
Common to many models of friction is the \emph{maximum dissipation principle}, which states that if $\FSet$ is the set of all possible forces acting at a contact, then the force at any given moment is one such that the power dissipated at the contact is maximized. For contact at a single point between two rigid bodies with relative velocity $\ve v$, this condition is realized as
\begin{equation}
	\ve f \in \arg\min_{\ve f' \in \mathcal{F}} \ve f' \cdot \ve v\;.
\end{equation}
For point contacts, $\mathcal{F}$ is often modeled as a cone by Coulomb friction, which enforces the following:
\begin{itemize}
\item For each contact, either the normal velocity is zero and the normal force is non-negative, or vice versa:
\begin{equation}\label{eq:VNLC}
\VNLC\;.
\end{equation}
\item The magnitude of the frictional force at the $\ArbIterB$th contact is bounded above by $\MuMat[\ArbIterB,\ArbIterB]\Fn[\ArbIterB]$. For sliding contacts, the frictional force has magnitude $\MuMat[\ArbIterB,\ArbIterB]\Fn[\ArbIterB]$ and is antiparallel to the sliding velocity. This behavior can be captured as complementarity constraints with the addition of a slack variable $\SVar$:
\begin{align}
&\VTLC\label{eq:VTLC}\;,\\
&\VSLC\label{eq:VSLC}\;.
\end{align}
\end{itemize}

\subsection{Friction at Contact Between Object and Surface}
Coulomb friction behavior cannot readily be applied to contacts where the normal force is not concentrated at a finite set of points. Zhou et al. \cite{ref:Zhou2018} therefore devise and experimentally validate a model that directly approximates the limit surface assuming a constant pressure distribution. They parameterize this behavior with a symmetric, scale-invariant, and strictly convex function $\FMMap(\FNet): \Real^3 \longrightarrow [0,\infty)$, defined over the space of body-axis friction wrenches $\FNet = \GVelJac (q_\theta)^T \FNet[\Obj]$. 
The physical meaning of this function is as follows:
\begin{itemize}
	\item The set of possible static friction wrenches is $\FSet = \{\FMMap(\FNet) \leq 1\}$.
	\item If the contact is sliding, then the maximum dissipation principle requires that $\FNet$ be on the boundary of $\FSet$, $\{\FMMap(\FNet) = 1\}$. Furthermore, $\GVel$ must lie in the normal cone of $\FSet$ at $\FNet$. 
\end{itemize}
As $\FSet$ is strictly convex, the latter condition is exactly satisfied by
\begin{equation}\label{eq:FMAlignment}
\exists k \geq 0, \GVel = k \nabla \FMMap(\FNet)\;.
\end{equation}
We note that \eqref{eq:FMAlignment} will also hold in the case of static friction with $k=0$.

\subsection{Friction Behavior as an LCP}
We examine the quasi-static model of Zhou et al. \cite{ref:Zhou2018}, which composes both point and surface contacts. From (\ref{eq:FMAlignment}), if $\FMMap(\FNet)$ has the ellipsoid form $\FNet^T \tilde{\FMMat} \FNet, \tilde{\FMMat} \succ \ZMat$,
\begin{equation}
	\GVel = k \tilde{\FMMat} \FNet = k \tilde{\FMMat}  \GVelJac (q_\theta)^T \FNet[\Obj]\;,
\end{equation}
\begin{equation}\label{eq:forcemotion}
	\RBVel[\Obj] = \GVelJac (q_\theta)\GVel  = k \GVelJac (q_\theta)\tilde{\FMMat} \GVelJac (q_\theta)^T \FNet[\Obj] = k \FMMat  \FNet[\Obj]\;,
\end{equation}
where $\FMMat = \GVelJac (q_\theta)\tilde{\FMMat} \GVelJac (q_\theta)^T \succ \ZMat$.  Assuming perfect velocity control (i.e. $\RBVel[\Manip] = \RBDVel$), (\ref{eq:VNLC})--(\ref{eq:VSLC}) reduce to $\LCP{\LCPMat}{\LCPVec(\RBDVel)}$, a generalization of (27) in \cite{ref:Zhou2018}, where
\begin{align}
\LCPMat &= \VLCPInfMat[\Obj]\label{eq:ZMatDef}\;,\\
\LCPVec(\RBDVel) &= \VLCPVec\;.\label{eq:ZVecDef}
\end{align}
Each $\LCPVar \in \LCPSol{\LCPMat}{\LCPVec (\RBDVel)}$ is a choice for $[k\Fvar ; \SVar]$ that complies with both the point and surface contact models. While solutions for $k$ and $\Fvar$ are not computed separately, their product is sufficient to calculate object velocity. For simplicity, we will denote $\LCPVar$ as $[\Fvar ; \SVar]$ for the rest of this paper. We also define $\FSet_{\RBDVel} = \{ \Fvar : \exists \SVar , [\Fvar ; \SVar] \in \LCPSol{\LCPMat}{\LCPVec(\RBDVel)} \}$, the set of feasible point contact forces for command velocity $\RBDVel$. For $\RBDVel = \ZVec$, $\FSet_{\RBDVel} = \FSet_{\ZVec}$ is the set of admissible \emph{internal forces} \cite{ref:Murray1994}, i.e. forces that map to zero net force and torque on the object ($\FNet[\Obj] = \ZVec$).
	
For non-quadratic descriptions of $\FMMap(\FNet)$, low accuracy solutions may be computed quickly by approximating $\FMMap$ as an ellipsoid. If higher accuracy solutions are required, one may solve a sequence of programs such that $\FMMat$ in the $\ArbIterB$th program is equal to the Hessian of $H$ evaluated at a solution of $(\ArbIterB-1)$th program.

\section{Finite Velocity Feedback Quasi-statics}

While the above formulation has been successful at simulating pushes \cite{ref:Zhou2018} and planning grasps under stochasticity \cite{ref:Zhou2017b}, the range of applications of this method is significantly limited due to undefined and ambiguous behaviors as displayed in Figure \ref{fig:PinchAndGraze}. Grasping and jamming commands may result in manipulator velocities that cannot be realized without penetrating (see Lemma 1, \cite{ref:Zhou2018}). Additionally, when the manipulator is commanded to graze the object (that is, the commanded velocities of all the contact points are parallel to the object boundary), there can be an infinite set of possible solutions which result in wildly different motions.

The source of this ambiguity is a critical assumption of the perfect velocity control model---that there exists a feedback controller internal to the manipulator that has high enough (essentially, infinite) gain to overcome any external disturbance.
From this perspective, grasping and jamming are undefined as they prescribe two unstoppable objects to oppose each other ("$\infty - \infty$"-like behavior). Additionally, while exact execution of a grazing maneuver may produce zero external disturbance for the manipulator controller to balance, a small perturbation of the commanded velocities towards the inward normal of the object surface could transform the command to an infeasible grasp or jam.
\begin{figure}[h]
\includegraphics[width=6cm]{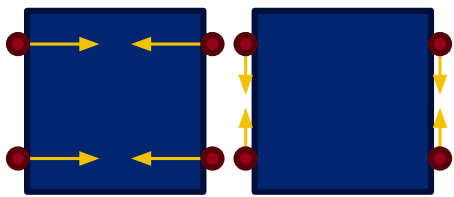}
\centering
\caption{Left: example of a configuration for which a quasi-static model assuming perfect velocity control does not yield any feasible solutions for $\RBVel[\Obj]$. Here, we have a square object (blue squares) with which a four-fingered manipulator (red circles) makes full contact. There are no generalized velocities for the object that would allow the fingers to move along their commanded trajectory (yellow arrows). Right: a configuration that yields ambiguous behavior. Depending on the normal force on the individual fingers, the object may remain stationary, slide upwards, or slide downwards.}\label{fig:PinchAndGraze}
\end{figure}

In order to resolve these issues, we explicitly interpret the quasi-statics of perfect velocity control as the limiting case of the quasi-statics of finite, stable, linear feedback. That is, we assume some relative feedback gain matrix $\FB \succ \ZMat$ and scaling factor $\FBScale$ such that the generalized force due to contact exerted on the manipulator is balanced by a feedback torque
\begin{equation}
	-\FNet[\Manip] = \frac{1}{\FBScale} \FB^{-1} (\RBDVel - \RBVel[\Manip])\;,
\label{eq:ManipForceBal}
\end{equation}
where we note that the gain is inversely proportional to $c$ and the scaling of $B$. While (\ref{eq:ZMatDef}) and (\ref{eq:ZVecDef}) assumed that the manipulator velocity directly tracked the desired velocity, $\RBVel[\Manip]=\RBDVel$, we instead solve (\ref{eq:ManipForceBal}) for $\RBVel[\Manip]$ and construct a new program $\LCP{\LCPMat(\FBScale\FB)}{\LCPVec(\RBDVel)}$ that accounts for the modified velocity:
\begin{align}
\RBVel[\Manip] &=  \RBDVel + \FBScale \FB \FNet[\Manip]\;,\\
\LCPMat(\FBScale\FB) &= \VLCPFinMat\;,\label{eq:QFMMatDef}\\
\LCPVec(\RBDVel) &= \VLCPVec\;.\label{eq:QFMVecDef}
\end{align}
While we have eliminated the perfect velocity control assumption, our formulation introduces no additional computational complexity, as our LCP is of equal dimension. We now prove that our model is well-behaved in the sense that solutions are guaranteed to exist (Theorem \ref{thm:VExist}); the manipulator velocity remains bounded as $c \rightarrow 0$ (Theorem \ref{thm:bddvel}); and solutions typically converge the perfect velocity control case when it is feasible (Theorem \ref{thm:VConv}).

\begin{theorem}\label{thm:VExist}
For all $(\FBScale, \FB, \RBDVel)$, if $\FBScale > 0$ and $\FB \succ 0$, $\LCPSol{\LCPMat(\FBScale\FB)}{\LCPVec(\RBDVel)} \neq \emptyset$.
\end{theorem}
\begin{proof}
Let $\LCPVar \geq \ZVec$. As all the indices of $\MuMat$ are non-negative, $\SVar^T \MuMat \Fn \geq 0$. We also have that $\FNet[\Obj]^T \FMMat \FNet[\Obj] \geq 0$ and $\FNet[\Manip]^T \FBScale\FB \FNet[\Manip] \geq 0$ as $\FMMat \succ 0$ and $\FBScale\FB \succ 0$. Therefore,
\begin{equation}
\LCPVar^T \LCPMat(\FBScale \FB) \LCPVar = \FNet[\Obj]^T \FMMat \FNet[\Obj] + \FNet[\Manip]^T \FBScale\FB \FNet[\Manip] + \SVar^T \MuMat \Fn \geq 0\;,
\end{equation}
and $\LCPMat(\FBScale\FB)$ is copositive. Now, suppose $\LCPVar \in \LCPSol{\LCPMat(\FBScale\FB)}{\ZVec}$. We must have
\begin{equation}
\FNet[\Obj]^T \FMMat \FNet[\Obj] + \FNet[\Manip]^T \FBScale\FB \FNet[\Manip] + \SVar^T \MuMat \Fn = 0\;,
\end{equation}
which implies $\FNet[\Manip] = \ZVec$. Therefore, $\LCPVar^T\LCPVec(\RBDVel) = \FNet[\Manip]^T\RBDVel = \ZVec$. The result follows from Proposition \ref{prop:LCPCopExist}.
\qed
\end{proof}

If the desired velocity $\RBDVel$ lies within a neighborhood of infeasible velocities (symptomatic of grasping or jamming behavior), then as we take $\FBScale \rightarrow 0$, the force balance in \eqref{eq:ManipForceBal} implies that $\FNet[\Manip]$, and therefore some of the individual finger forces, will grow unboundedly.
However, the net manipulator velocity error $\RBVel[\Manip] - \RBDVel = \FBScale \FB \FNet[\Manip] = \FB \Jac[\Manip]^T(\FBScale \Fvar)$ remains bounded due to the boundedness of $\FBScale \FNet[\Manip]$:
\begin{theorem}\label{thm:bddvel}
Let $\FB \succ \ZVec$. There exists $\ArbRConB(\FB) \in \Real$ such that for all $\FBScale > 0$, for all $\LCPVar \in \LCPSol{\LCPMat(\FBScale\FB)}{\LCPVec(\RBDVel)}$, $\FBScale \Enorm{\FNet[\Manip]} \leq \ArbRConB(\FB) \Enorm{\RBDVel} $.
\end{theorem}
\begin{proof}
Let $\LCPVar \in \LCPSol{\LCPMat(\FBScale \FB)}{\LCPVec(\RBDVel)}$. From Definition \ref{def:LCPDef}, we have that 
\begin{align}
\LCPVar &\geq 0\;,\\
\FNet[\Obj]^T \FMMat \FNet[\Obj] + \FNet[\Manip]^T \FBScale \FB \FNet[\Manip] + \SVar^T \MuMat \Fn + \FNet[\Manip]^T\RBDVel &= 0\;.\label{eq:QFMComp}
\end{align}
Let $\ArbEVal_{min} > 0$ be the minimum eigenvalue of $\FB$. Noting $\FMMat \succ \ZVec$, we have that
\begin{align}
\FBScale \FNet[\Manip]^T \FB \FNet[\Manip] + \FNet[\Manip]^T\RBDVel & \leq 0\;,\\
\ArbEVal_{min} \FBScale \Enorm{\FNet[\Manip]}^2+ \FNet[\Manip]^T\RBDVel & \leq 0\;,\\
\ArbEVal_{min} \FBScale \Enorm{\FNet[\Manip]}^2 & \leq \Enorm{\FNet[\Manip]}\Enorm{\RBDVel}\;, \\
\FBScale \Enorm{\FNet[\Manip]} & \leq \dfrac{1}{\ArbEVal_{min}}\Enorm{\RBDVel} \;.
\end{align}
\qed
\end{proof}

Intuitively, if individual finger forces grow without bound, yet the net force remains bounded, then there must be a ``canceling out'' effect. More precisely, the portion of $\Fvar[\ArbIter]$ that experiences this growth must be an internal force:

\begin{theorem}\label{coro:nulldisp}
 For all sequences $\Seq{\LCPVar}{\ArbIter}$ and $\Seq{\FBScale}{\ArbIter}$ such that $\FBScale[\ArbIter] \rightarrow 0$ and $\LCPVar_\ArbIter \in \LCPSol{\LCPMat(\FBScale[\ArbIter] \FB)}{\LCPVec(\RBDVel)}$, $\FBScale[\ArbIter]  \Fvar[\ArbIter] \rightarrow {\FSet}_{\ZVec}$
\end{theorem}
\begin{proof}
Let $\tilde{\FSet}_{\ZVec} = \{ \Fvar \geq \ZVec : \Jac[\Obj]^T \Fvar = \ZVec \wedge \MuMat \Fn - \EMat^T \Ft \geq \ZVec \}$. $\tilde{\FSet}_{\ZVec} \subseteq {\FSet}_{\ZVec}$ by constructing $\tilde{\LCPVar} = [\tilde{\Fvar[]}; \ZVec] \in \LCPSol{\LCPMat(\ZMat)}{\ZVec}$ from $\tilde{\Fvar[]} \in \tilde{\FSet}_{\ZVec}$. It is therefore sufficient to show $\FBScale[\ArbIter]  \Fvar[\ArbIter] \rightarrow \tilde{\FSet }_{\ZVec}$.

We first note that $\FBScale[\ArbIter] {\LCPVar}_{\ArbIter} \geq \ZVec$, and $\MuMat \FBScale[\ArbIter] \Fn[\ArbIter] - \EMat^T \FBScale[\ArbIter] \Ft[\ArbIter] \geq \ZVec$ follow directly from $\LCPVar_\ArbIter \in \LCPSol{\LCPMat(\FBScale[\ArbIter] \FB)}{\LCPVec(\RBDVel)}$.  Multiplying (\ref{eq:QFMComp}) by $\FBScale[\ArbIter]^2$, we have
\begin{equation}
\FBScale[\ArbIter]\FNet[\Obj,\ArbIter]^T \FMMat \FNet[\Obj,\ArbIter]\FBScale[\ArbIter] + \FBScale[\ArbIter]^3 \FNet[\Manip,\ArbIter]^T \FB \FNet[\Manip,\ArbIter] + \FBScale[\ArbIter] \SVar[\ArbIter]^T \MuMat \Fn[\ArbIter] \FBScale[\ArbIter] + \FBScale[\ArbIter]^2 \FNet[\Manip,\ArbIter]^T\RBDVel = 0\;.
\end{equation}
In the limit, the second and fourth terms in this equation vanish due to the boundedness of $\FBScale[\ArbIter] \FNet[\Manip,\ArbIter]$, rendering
\begin{equation}
\FBScale[\ArbIter]\FNet[\Obj,\ArbIter]^T \FMMat \FNet[\Obj,\ArbIter]\FBScale[\ArbIter] + \FBScale[\ArbIter] \SVar[\ArbIter]^T \MuMat \Fn[\ArbIter] \FBScale[\ArbIter] \rightarrow 0\;.
\end{equation}
As $\FMMat \succ 0$, $\FBScale[\ArbIter]\FNet[\Obj,\ArbIter] = \Jac[\Obj]^T \FBScale[\ArbIter] \Fvar[\ArbIter] \rightarrow \ZVec $, and therefore $\FBScale[\ArbIter]  \Fvar[\ArbIter] \rightarrow \tilde{\FSet }_{\ZVec}$.
\qed

\end{proof}

$\FSet_{\ZVec}$ can therefore be considered a basis from which one can generate the errors in the manipulator velocity; for $\FBScale$ sufficiently close to $0$, there exists a $\Di \in \FSet_{\ZVec}$, such that $\RBVel[\Manip] - \RBDVel \approx \FB \Jac[\Manip]^T \Di$. When there are no non-zero internal forces, the manipulator displacement approaches zero, and the solution of the finite linear feedback model approaches the behavior for perfect velocity control:

\begin{theorem}\label{thm:VConv}
Suppose that $\LCPMat(\ZMat)$ is chosen such that $\FSet_{\ZVec} = \{ \ZVec \}$. For all sequences $\Seq{\LCPVar}{\ArbIter}$,$\Seq{\FBScale}{\ArbIter}$ such that $\FBScale[\ArbIter] \rightarrow 0$ and $\LCPVar_\ArbIter \in \LCPSol{\LCPMat(\FBScale[\ArbIter] \FB)}{\LCPVec(\RBDVel)}$, then $\Fvar[\ArbIter] \rightarrow \FSet_{\RBDVel}$. 
\end{theorem}
\begin{proof}
Assume the contrary, so that there exists $\Seq{\FBScale}{\ArbIter} \rightarrow 0$ and $\Seq{\LCPVar}{\ArbIter}$ bounded away from $\LCPSol{\LCPMat(\ZMat)}{\LCPVec(\RBDVel)}$ such that $\LCPVar_\ArbIter \in \LCPSol{\LCPMat(\FBScale[\ArbIter] \FB)}{\LCPVec(\RBDVel)}$. Letting ${\RBVel[]}_{\ArbIter} = \RBDVel + {\ArbRCon}_{\ArbIter} \FB \FNet[\Manip,\ArbIter]$, we have that ${\LCPVar}_{\ArbIter} \in \LCPSol{\LCPMat(\ZMat)}{\LCPVec({\RBVel[]}_{\ArbIter})}$. ${\ArbRCon}_{\ArbIter} {\Fvar}_{\ArbIter} \rightarrow \ZVec$, given by Theorem \ref{coro:nulldisp}, implies ${\RBVel[]}_{\ArbIter} \rightarrow \RBDVel$. For all $\ArbIter$, we observe the complementarity condition
\begin{equation}\label{eq:nthcomp}
\FNet[\Obj,\ArbIter]^T\FMMat\FNet[\Obj,\ArbIter] + \SVar[\ArbIter] \MuMat \Fn[\ArbIter] + \FNet[\Manip,\ArbIter]^T{\RBVel[]}_{\ArbIter} = \ZVec\;,
\end{equation}
which implies that $\SVar[\ArbIter]\MuMat \Fn[\ArbIter]$ is bounded for bounded $\Enorm{\Fvar[\ArbIter]}$ .
As $\FSet_{\ZVec} = \{ \ZVec \}$, for all $\LCPVar \in \LCPSol{\LCPMat(\ZMat)}{\ZVec}$, we have $\LCPVar^T \LCPVec(\RBDVel) = 0$. Therefore by Proposition \ref{prop:LCPCopExist}, $\FSet_{\RBDVel}$ and $\LCPSol{\LCPMat(\ZMat)}{\LCPVec(\RBDVel)}$ are non-empty.

If there were a subsequence ${\LCPVar}_{\ArbIter_{\ArbIterB}}$ such that ${\Fvar}_{\ArbIter_{\ArbIterB}}$ were bounded, then from (\ref{eq:nthcomp}) and the non-emptiness of $\FSet_{\RBDVel}$,  ${\Fvar}_{\ArbIter_{\ArbIterB}}$ would have a limit point in $\FSet_{\RBDVel}$, violating our assumptions. Therefore, we must have $\Enorm{{\Fvar}_{\ArbIter}} \rightarrow \infty$. Similar to Theorem \ref{coro:nulldisp}, by dividing (\ref{eq:nthcomp}) by $\Enorm{\Fvar[\ArbIter]}^2$,
\begin{equation}
\frac{1}{\Enorm{\Fvar[\ArbIter]}}\FNet[\Obj,\ArbIter]^T \FMMat \FNet[\Obj,\ArbIter]\frac{1}{\Enorm{\Fvar[\ArbIter]}}  \rightarrow 0 \;,
\end{equation}
and thus $\tilde{\Fvar[]}_{\ArbIter} = \frac{{\Fvar[]}_{\ArbIter}}{\Enorm{\Fvar[\ArbIter]}} \rightarrow \FSet_{\ZVec} = \{ \ZVec \}$. But $\Enorm{\tilde{\Fvar[]}_{\ArbIter}} = \frac{\Enorm{\Fvar[\ArbIter]}}{\Enorm{\Fvar[\ArbIter]}}= 1$, a contradiction.
\qed
\end{proof}

We note that $\FSet_{\ZVec} = \{ \ZVec \}$ implies the current configuration is not a force-closure \cite{ref:Murray1994}. In practice, we expect the set of non-force-closure configurations that have $\FSet_{\ZVec} \neq \{ \ZVec \}$ to be small. We also expect perfect velocity control models to behave poorly during force-closure, as they exhibit grasping behavior. Therefore, this model will behave more realistically in a variety of scenarios, with minimal accuracy loss for prehensile commands that perfect velocity control models handle well already.

\section{Time-Stepping Scheme}

Despite the guarantee (Theorem~\ref{thm:VExist}) that (\ref{eq:QFMMatDef}) and (\ref{eq:QFMVecDef}) provide feasible instantaneous velocity solutions, embedding the LCP into common ODE schemes is an incomplete approach, as the resulting velocities will be discontinuous whenever contact is initiated between the manipulator and objects.
Anitescu and Potra \cite{ref:Anitescu2002} resolved a similar issue in their formulation for 3D multibody simulation by root-finding the first sub-time-step impact. While a similar modification could be applied to our LCP, the ability to resolve sub-time-step impacts in a single LCP would be beneficial. To that end, we take inspiration from Stewart and Trinkle \cite{ref:Stewart1996}, and instead formulate an alternative LCP that explicitly models the positions of the manipulator and object at the end of the time-step.

To construct this new program, instead of solving for the force at a particular time $t$, we solve for the the net normal and tangential impulse between times $t$ and $t+h$ at each contact, $\DFn$ and $\DFt$. Using the superscripts $\TStart$ and $\TEnd$ to denote values calculated at the beginning and end of this interval, we linearize $\Dist$ as
\begin{equation}
	\DDist = \Dist^\TEnd - \Dist^\TStart \approx \PVDiff{\Dist[]}{\RBPos[\Obj]}^\TStart \DRBPos[\Obj] + \PVDiff{\Dist[]}{\RBPos[\Manip]}^\TStart \DRBPos[\Manip] = \NJac[\Obj]^\TStart \DRBPos[\Obj] + \NJac[\Manip]^\TStart \DRBPos[\Manip]\;,\label{eq:DDist}
\end{equation}
and we make a first-order approximation of $\DRBPos[\Obj]$ and $\DRBPos[\Manip]$:
\begin{align}
\DRBPos[\Obj] = \RBPos[\Obj]^\TEnd - \RBPos[\Obj]^\TStart &\approx h \RBVel[\Obj]^\TEnd \approx \FMMat \parenEnc{{\NJac[\Obj]^{\TStart T}}  \DFn + {\TJac[\Obj]^{\TStart T}}  \DFt}\;,\\
\DRBPos[\Manip] = \RBPos[\Manip]^\TEnd - \RBPos[\Manip]^\TStart &\approx h \RBVel[\Manip]^\TEnd \approx \FB \parenEnc{{\NJac[\Manip]^{\TStart T}}  \DFn + {\TJac[\Manip]^{\TStart T}}  \DFt} + h\RBDVel^\TStart\;.
\end{align}
$\DFn$ and $\DFt$ are subject to the complementarity constraints
\begin{align}
&\PNLC\;,\\
&\PTLC\;,\\
&\PSLC\;.
\end{align}
Arranging into the standard format, we arrive at $\LCP{\tilde{\LCPMat}(\FB)}{\tilde{\LCPVec}(\RBDVel^\TStart)}$, where 
\begin{align}
\tilde{\LCPMat}(\FBScale \FB) &= \LCPMat(\FBScale \FB)^\TStart\;,\\
\tilde{\LCPVec}(\RBDVel^\TStart)  &= \LCPVec(h\RBDVel^\TStart)^\TStart + \begin{bmatrix} \Dist \\ \ZVec \end{bmatrix}\;.
\end{align}
With the exception of the added $\Dist^\TStart$ in $\tilde{\LCPVec}$, this program is identical to an instantiation of the velocity formulation defined in (\ref{eq:QFMMatDef}) and (\ref{eq:QFMVecDef}). Noting that $\Dist^\TStart \geq \ZVec$ for any feasible initial condition, a trivial extension of Theorem \ref{thm:VExist} guarantees existence of solutions for all feasible initial conditions. This is a significant improvement over existing time-stepping schemes for the full dynamic behavior, as it circumvents both the expensive root-finding subroutine of \cite{ref:Anitescu2002} and the non-existance issues in \cite{ref:Stewart1996}. However, it is important to note that if $\Dist (\RBPos)$ is non-linear, the linearization in (\ref{eq:DDist}) does not guarantee that the true value of $\Dist^\TEnd$ will be feasible (positive), even if $\Dist^\TStart$ is feasible.
In these cases, similar to \cite{ref:Stewart1996}, one may rectify this issue by solving a sequence of problems in a fixed-point iteration scheme, linearizing the problem about the best current estimate for $(\RBPos[\Obj]^\TEnd,\RBPos[\Manip]^\TEnd)$. This iteration can be conducted at the same time as the iteration for a non-ellipsoidal $\FMMap$. We additionally note that setting $\FB = \ZMat$ gives a time-stepping reformulation of the program in \cite{ref:Zhou2018}.	

\section{Examples}
We provide a few examples to illustrate the capabilities  and accuracy of our approach. Three examples from our open source MATLAB library\footnote{\url{https://github.com/mshalm/quasistaticGrasping}} are provided in conjunction with a video depiction\footnote{\url{https://www.youtube.com/watch?v=1wAH5o3OLck}}. Additionally, we compare the output of our model to a fully dynamic, compliant simulation using Drake \cite{ref:Tedrake2016}. All LCPs associated with our model are solved with PATH \cite{ref:Dirske1995}.
\subsection{Pushing with Two Fingers}
\label{subsec:twoFingerPush}
\begin{figure}[h]
\vspace*{-4mm}
\hspace*{0mm}
\includegraphics[width=12cm]{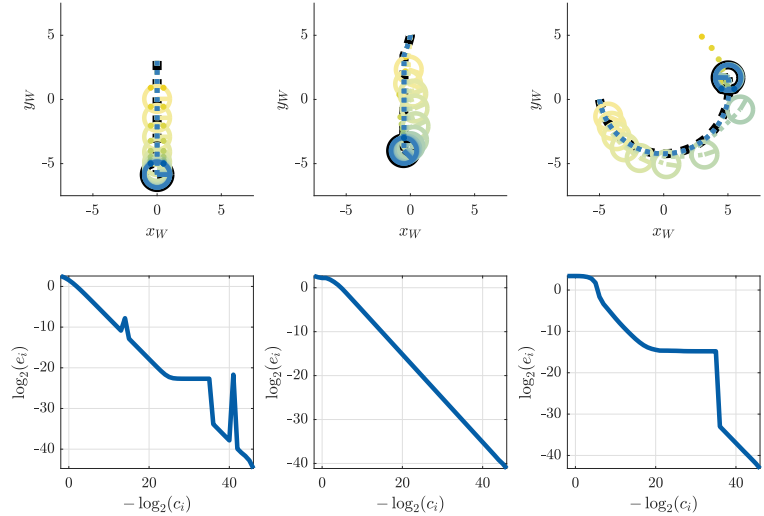}
\vspace*{-6mm}
\centering
\caption{\textbf{Top}: three maneuvers are executed at varying feedback intensities: a symmetric push, an asymmetric push, and a semicircular push. The black dotted lines and circles represent the center of mass trajectory and finale pose for $\RBPos[\Obj,\infty]$. The finite feedback trajectories are respresented similarly, with the color shifting from yellow to blue as the feedback gain ${\ArbRCon}_{\ArbIter}^{-1}$ increases. \textbf{Bottom}: log-log plot comparing $\FBScale[\ArbIter]$ and $e_{\ArbIter}$ for the corresponding commands.}
\label{fig:TripleConvergence}
\end{figure}

We first consider a flat disk of radius $1\,\textrm{m}$ pushed by two fully-actuated point fingers. The coefficient of friction between the fingers and the disk is set at $1$, and the force-motion map is set as $\FMMap(\FNet[\Obj]) = \FNet[\Obj]^T\FNet[\Obj]$ such that $\nabla^2 \FMMap = \FMMat = \Eye_3$. We consider the configuration $\RBPos[\Manip] = \begin{bmatrix} q_{x,1} & q_{y,1} & q_{x,2} & q_{y,2} \end{bmatrix}^T$ which directly represents the $x-y$ coordinates of each finger in a fixed frame. We assume each coordinate is controlled independently with equal gain, such that $\FB = \Eye_4 \, \textrm{m s}^{-1}\textrm{ N}^{-1}$.

We now empirically evaluate the validity of Theorem \ref{thm:VConv}. Analytically, we expect that as we take ${\ArbRCon}_{\ArbIter} \rightarrow 0$, finite feedback simulation should converge to perfect velocity command tracking. We simulate three motions over $t \in (0,10)\,\textrm{s}$ at $40\,\textrm{Hz}$ for various feedback scaling terms ${\ArbRCon}_{\ArbIter}$, and plot the corresponding object trajectories $\RBPos[\Obj,\ArbIter](t)$ in Figure \ref{fig:TripleConvergence}. We compare the final configuration of each finite feedback trajectory with $\RBPos[\Obj,\infty](t)$, the trajectory resultant from executing the same manipulator command with a perfect velocity control. For sufficiently high gains, we see a linear relationship between the log of the feedback scaling term and that of $e_\ArbIter = \Enorm{\RBPos[\Obj,\ArbIter](10) - \RBPos[\Obj,\infty](10)}$, the error in the final pose.
\begin{figure}[h]
\centering
\vspace*{-4mm}
\subfloat{\includegraphics[width=7cm]{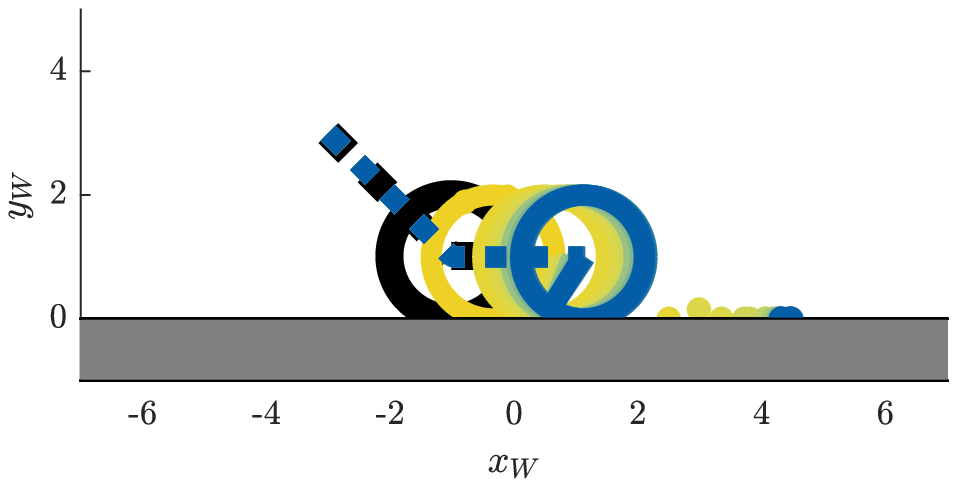}}\\
\hspace*{-11mm}
\subfloat{\includegraphics[width=14cm]{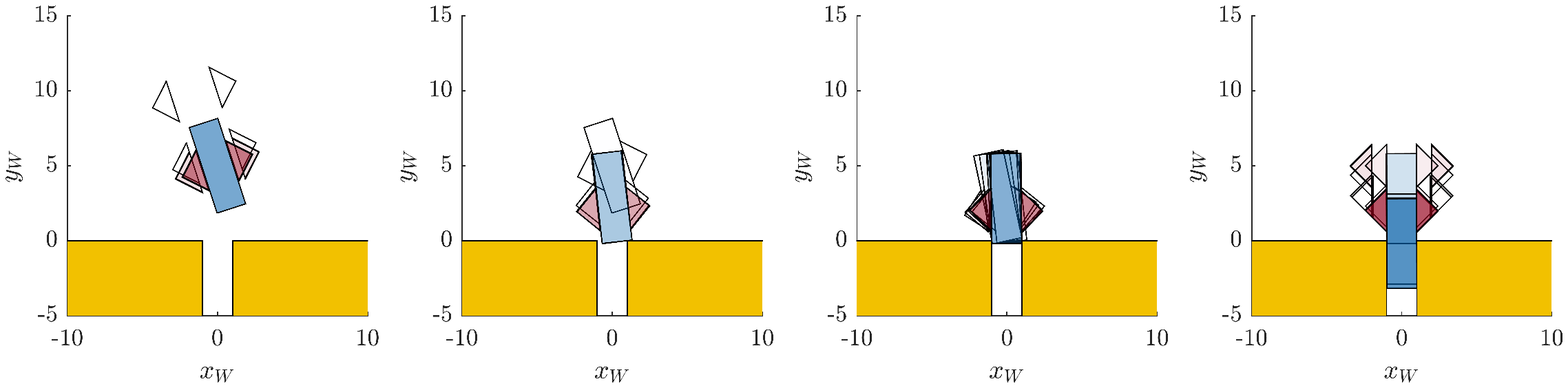}}
\vspace*{-3mm}
\caption{\textbf{Top}: motion on the same system displayed in Figure \ref{fig:TripleConvergence}, except with the addition of a wall located at $y_W = 0$, is shown in a similar manner. As the commanded motion involves squeezing the object against the wall, the perfect velocity tracking program (black) terminates at the first time-step at which the squeezing occurs. \textbf{Bottom}: a single trajectory for the polygonal peg-in-hole system segmented into four phases. \textbf{Far left}: grasping. \textbf{Center left}: jamming against the right side of the slot. \textbf{Center right}: twisting into the slot. \textbf{Far right}: final insertion.}
\label{fig:WallSlide}
\end{figure}

However, the low-gain performance, particularly on the semicircular command, showcases an important nuance in the convergent behavior described in Theorem \ref{thm:VConv}. While the method may converge over a single time-step, differences in individual time-steps may accumulate such that there is a change in the contact mode during the motion. In these low-gain cases, the fingers tend to slide off of the sides of the object, leaving the object behind its intended goal.

We now show our formulation's ability to simulate through jamming motions that perfect velocity control models cannot capture. We attempt to push the object into and roll along a wall. We use identical simulation parameters, and plot the results in Figure \ref{fig:WallSlide}. We can see that the perfect velocity control simulation terminates at the collision with the wall, while the finite gain formulation not only captures the collision, but also still exhibits convergent behavior thereafter.

\subsection{Polygonal Peg-in-hole}
We apply our method to a peg-in-hole problem, a more complex task requiring initiation and release of several contacts and gripping motions. We simulate a new system consisting of a thin rectangular peg, two triangular manipulator fingers, and a slot into which the peg is fit. The slot has a $1\%$ tolerance on the width of the peg, the relative feedback gains are set to $\FB = \Eye_6$ (in $\textrm{m s}^{-1}\textrm{ N}^{-1}$ units for linear terms and $\textrm{m}^{-1} \textrm{ s}^{-1}\textrm{ N}^{-1}$ for rotational terms), and the feedback scaling is set to $\FBScale = 0.01$. In a hand-designed trajectory, the peg catches the corner of the slot in the initial insertion attempt, after which the manipulator reorients and successfully inserts it. For each time step, only bodies close enough to make contact are considered, allowing the LCPs to be kept small. We set our time step to $50\,\text{ms}$, for which PATH is able to compute solutions in  $1.51\,\text{ms}$ on average. The trajectory is  displayed in Figure \ref{fig:WallSlide}.
\subsection{Comparison to Full Dynamics}
	We now evaluate the similarity of our model to Drake \cite{ref:Tedrake2016}, a framework for simulation of rigid-body dynamics with contact. We simulate a square object of mass $.01\,\textrm{kg}$ and side length $0.4\,\textrm{m}$ in contact with four manipulators, each consisting of two $1\,\textrm{m}$ long links with a circular contact at the end of the second link. The manipulators pinch the object, which is particularly numerically challenging to simulate. While the velocity commands are symmetric, each manipulator is driven with different feedback gains, causing the object to move in the $+y_W$ direction and spin in response to the pinch. The dynamic and quasi-static simulations exhibit qualitatively similar behavior, shown in Figure \ref{fig:DynComp}.
\begin{figure}[h]
\vspace*{0mm}
\hspace*{-6mm}
\includegraphics[width=12cm]{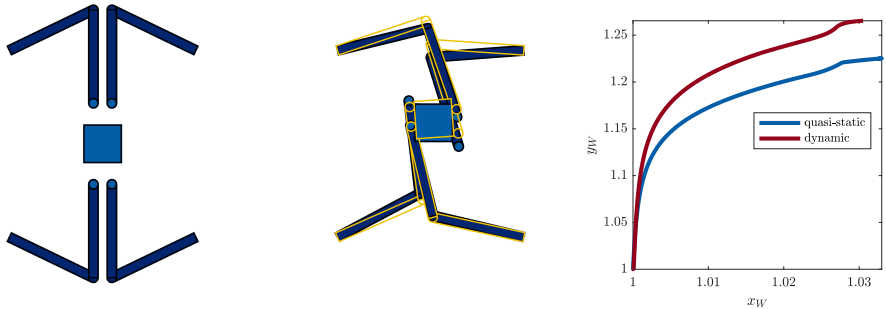}
\centering
\vspace*{-3 mm}
\caption{\textbf{Left}: the initial condition for both object pinching trajectories. \textbf{Center}: An overlay of the final conditions for the dynamic (yellow wireframe) and quasi-static (solid blue) trajectories. \textbf{Right}: A comparison of the quasi-static and dynamic trajectories of the center of mass of the object. The object moves slightly more in the $y_w$ direction in the dynamic trajectory than it does under quasi-statics.}
\label{fig:DynComp}
\end{figure}

\section{Discussion and Future Work}
We have presented a method for generating motion using a quasi-static model for planar manipulation. By explicitly modeling manipulator feedback behavior, our method is able to synthesize physically-grounded motion in the face of infeasible velocity commands. Despite this added capability, our method is as computationally efficient as previous methods. We have validated our model theoretically and empirically by proving convergence to an established model.

An interesting property of many quasi-static manipulation models is that they allow decoupling of manipulator dynamics from object motion. However, we assert that Theorem \ref{coro:nulldisp} implies that controller choice inherently effects grasping equilibrium and is therefore an unavoidable component of a physically-grounded grasping model.
We also view the decoupling of the relative feedback gains from the overall feedback intensity as essential for the accuracy of our model. As the force scaling term $k$ of (\ref{eq:FMAlignment}) is not explicitly determined, even given true manipulator gains, one cannot synthesize $\LCPMat(\FBScale\FB)$ with complete accuracy. However, in light of Theorems \ref{coro:nulldisp} and \ref{thm:VConv}, if the controller has high enough gain, a small enough choice for $\FBScale$ will produce accurate results. In these cases, $\FBScale$ is a numerical term rather than something physically meaningful; one should choose $\FBScale$ as small as possible without degrading numerical precision in the construction of $\LCPMat(\FBScale\FB)$.

One might suspect that embedding a high-gain controller into the model induces stiff behaviors, as the corresponding full dynamics are stiff. However, our quasi-static assumptions happen to eliminate this behavior. As Theorem \ref{thm:VConv} proves convergence to the perfect velocity control model in most cases and Theorem \ref{thm:bddvel} proves that velocities are bounded by $\RBDVel$ otherwise, our feedback terms does not add significant stiffness. Some numerical precision may however be lost for small $\FBScale$ due to round-off error or poor conditioning of the associated LCPs. In future work, we will conduct a quantitative analysis of this behavior.

We do not expect solutions to our LCPs to be unique, as non-uniqueness is pervasive in complementarity-based contact models \cite{ref:Brogliato1999,ref:Stewart2000}. While our model does construct a unique mapping between $\FNet[\Manip]$ and $\RBVel[\Manip]$,  there are unmet assertions required for the mapping between $\RBVel[\Obj]$ and $\RBDVel$ to be unique.  However, it does disambiguate some grazing cases (such as in Figure \ref{fig:PinchAndGraze}).

Future extension of this result to 3D motion poses significant challenges. In the 2D case, the contact between the object and the surface below generates a unique map from contact forces to object motion. In 3D motion, the manipulator must instead counteract gravity. Furthermore, quasi-static modeling cannot realistically capture certain actions; for instance, dropping the object may either result in a lack of solution, or in a $\DRBPos[\Obj]$ large enough to make linearization of $\Dist$ inaccurate.
Possible applications of this model include controller synthesis through sums-of-squares based Lyapunov analysis and model predictive control, as well as planning via trajectory optimization methods.

\section{Acknowledgements}
This material is based upon work supported by the National Science Foundation under Grant No. CMMI-1830218. The authors thank Joury Abdeljaber for her work on numerical experiments.
\bibliographystyle{splncs03}
\bibliography{bibl}
\end{document}